\documentclass{article} 
\usepackage{iclr2027_conference,times}

\usepackage{amsmath,amssymb,amsthm}

\newtheorem{theorem}{Theorem}
\newtheorem{lemma}{Lemma}
 
\usepackage{amsmath,amsfonts,bm}

\def\eqref#1{equation~\ref{#1}}

\def\1{\bm{1}}

\DeclareMathAlphabet{\mathsfit}{\encodingdefault}{\sfdefault}{m}{sl}
\SetMathAlphabet{\mathsfit}{bold}{\encodingdefault}{\sfdefault}{bx}{n}

\usepackage{graphicx}
\usepackage{booktabs}
\usepackage{xcolor}
\usepackage{hyperref}
\usepackage{url}
\usepackage{microtype}
\usepackage{enumitem}
\usepackage{algorithm}
\usepackage{algpseudocode}
\usepackage{multirow}
\usepackage{wrapfig}

\newcommand{\ours}{\textsc{care}}
 
\title{CARE: Certifying Acceleration for Vision-Language-Action Inference}

\author{
Rui Liu$^{1}$,
Tong Zheng$^{1}$, 
Jindong Gu$^{2}$, 
Zhipeng Wang$^{3}$ \\
$^{1}$University of Maryland, College Park, 
$^{2}$University of Oxford,
$^{3}$Google
}

\iclrfinalcopy 
\begin{document}

\maketitle

\vspace{-8pt}
\begin{abstract}

\vspace{-6pt}

While vision-language-action (VLA) models have advanced rapidly, running them at every control step remains expensive. Prior work accelerates VLA inference using techniques like action chunking and visual-token pruning, typically evaluating based on latency and average task success. However, acceleration may discard information and break tasks the original policy would solve, a risk hidden by average metrics. Measuring these failures is challenging because action deviations compound over closed-loop trajectories, meaning task failure is only observable across full episodes. We therefore define an acceleration-induced failure via paired rollouts from identical initial conditions, tracking when the reference succeeds but the accelerated policy fails. To manage this, we introduce \ours{}, an approach for certified accelerator selection. \ours{} uses paired rollouts on a calibration set to provide finite-sample guarantees that acceleration-induced failure risk stays below a user-specified budget. It deploys the fastest certified candidate, falling back to the reference if none qualify. By relying only on terminal outcomes and measured compute, \ours{} applies unchanged across diverse acceleration mechanisms, while sequential testing and failure-triggered reference rollouts keep certification affordable. On four LIBERO suites with OpenVLA-OFT, \ours{} certifies $9.0$--$10.8\times$ speedups while guaranteeing (at $95\%$ confidence) that at least $85.8\%$ of reference-solved episodes are preserved. Under tight budgets, selectors without guarantees exceed the budget in up to $75\%$ of trials, whereas \ours{} stays within budget and its sequential form uses $78.9\%$ fewer rollouts than exhaustive evaluation. \ours{} further generalizes to flow-step reduction for $\pi_{0.5}$, and to Qwen3.5-9B and Llama-3.1-8B agents in Crafter.

\end{abstract}


\vspace{-10pt}
\section{Introduction}
\label{sec:intro}

\vspace{-6pt}
Vision-language-action (VLA) models have rapidly advanced multimodal decision-making by mapping visual observations and language instructions to actions~\citep{kim2024openvla, black2024pi0}. Their inference cost, however, remains a practical bottleneck: repeatedly running a large model in a closed-loop controller can make execution slow. A growing literature therefore reduces VLA inference cost through action chunking~\citep{kim2025fine}, visual-token pruning~\citep{chen2024image, liu2025bridging}, feature caching~\citep{xu2026vla}, dynamic inference~\citep{yue2024deer}, and fewer refinement steps for flow-matching policies~\citep{black2024pi0, intelligence2025pi05}. These approaches differ substantially in mechanism, but are typically evaluated using the same two quantities: inference cost and average task success.

These metrics describe how fast a policy is and how often it succeeds overall, but not what acceleration breaks. Acceleration may discard visual information, to which multimodal models are known to be
sensitive~\citep{liu2026reinforcing}, reuse stale features, or reduce re-planning frequency. A faster policy can therefore fail on tasks that the original policy would have solved. Average success does not isolate these regressions: failures shared by both policies count against the accelerated policy, while tasks newly solved by the accelerated policy can compensate for tasks that acceleration destroys. Two policies can consequently have similar average success while behaving very differently. This raises a key deployment question we ask:
\emph{how much can we accelerate a policy while bounding how often acceleration breaks tasks the reference already solves?}

We formalize this question through \emph{acceleration-induced failure}: from the same task and initial scene, the reference policy, run at full compute, succeeds while the accelerated policy fails. This event is defined over two complete closed-loop rollouts. Once acceleration changes an action, the next state changes as well, and every subsequent observation and decision may diverge from the reference trajectory.
We show that the resulting episode-level risk cannot in general be recovered from the two policies' marginal success rates or from per-decision measurements such as confidence or action coverage (Appendix~\ref{app:identification}), so it must be measured on paired, complete episodes. This makes certification challenging in two ways. Statistically, the risk must be bounded from a finite number of episodes, and choosing among many candidate accelerators creates many chances to accept an unsafe one by luck. Computationally, every evaluation requires closed-loop rollouts, including rollouts of the expensive reference policy that
acceleration is meant to avoid.

\begin{figure}[t]
\centering
\includegraphics[width=0.95\linewidth]{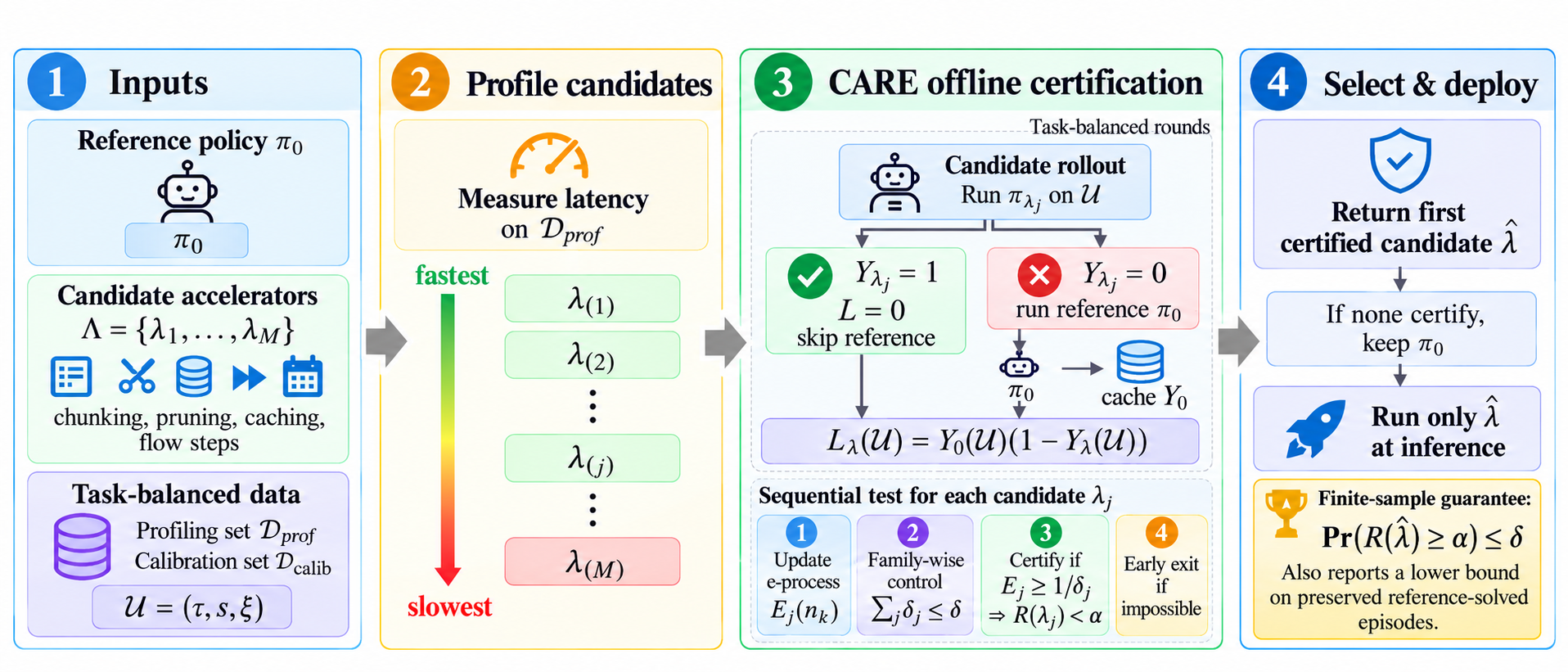}
\vspace{-8pt}
\caption{\textbf{Overview of \ours{}.} Given a reference policy $\pi_0$ and accelerated candidates $\Lambda=\{\lambda_1,\ldots,\lambda_M\}$, \ours{} first profiles candidate latency on a disjoint profiling set and orders candidates from fastest to slowest. On a separate calibration set, candidates are evaluated sequentially in task-balanced rounds, with each round containing the same number of episodes from every task. For each paired calibration episode, the reference and accelerated policy start from the same evaluation unit $U$, and \ours{} records the episode-level loss $L_\lambda(U)=Y_0(U)\bigl(1-Y_\lambda(U)\bigr)$, which equals one only when the reference succeeds and the accelerated policy fails. An anytime-valid test with family-wise error control certifies candidates while allowing early stopping. To reduce calibration cost, the reference policy is executed only when a candidate fails, and its outcome is cached for reuse. \ours{} returns the first, and therefore fastest, certified candidate; if none can be certified, it retains $\pi_0$. The returned policy satisfies the finite-sample deployment guarantee without adding overhead at inference.}
\label{fig:care}
\vspace{-13pt}
\end{figure}


We introduce \ours{}, an approach for certified accelerator selection. \ours{} runs once, offline, before deployment. Given a reference policy, a family of acceleration configurations, and a user-defined risk budget, it executes paired rollouts on a  calibration set, records whether each rollout succeeds, and constructs a finite-sample certificate for each candidate's acceleration-induced failure risk. Concretely, it leverages Learn-then-Test~\citep{angelopoulos2025learn} with finite-sample concentration bounds \citep{hoeffding1963probability,bentkus2004hoeffding} to test whether each candidate's risk is below the budget while controlling errors across the candidate family, and returns the fastest certified candidate, or retains the reference when the data do not support any acceleration. At deployment, only the selected policy runs. \ours{} is also model and accelerator-agnostic: it treats each candidate as a black-box policy and requires only terminal outcomes and measured compute, so the same procedure applies to action chunking, token pruning, feature caching, and their combinations. 


Beyond this formulation, \ours{} also makes certification itself adaptive and
computationally efficient, as illustrated in Figure~\ref{fig:care}. The candidate family is ordered by latency measured on a separate profiling set, and \ours{} evaluates candidates from fastest to slowest, stopping at the first one it certifies, so slower candidates need not be evaluated at all. For each candidate, it applies anytime-valid sequential tests on task-balanced rounds of scenes, stopping as soon as the evidence certifies the candidate or rules it out, with a guarantee that remains valid even though tasks differ widely in risk. It further avoids unnecessary reference executions: if a candidate succeeds on a scene, it cannot have caused an acceleration-induced failure there, so the reference is run only when a candidate fails, and its outcome is cached for later candidates. These mechanisms preserve the finite-sample guarantee while substantially reducing the rollout cost of certification.

 
We evaluate \ours{} with OpenVLA-OFT \citep{kim2025fine} on four LIBERO \citep{liu2023libero} suites, using candidates that combine action chunking with visual-token pruning and feature caching. \ours{} achieves $9.0$--$10.8\times$ speedups over one-step replanning while guaranteeing, simultaneously at $95\%$ confidence, that the selected policies preserve at least $85.8\%$ of reference-solved episodes. Under tight risk budgets, common selectors, such as choosing the fastest candidate, matching average success, or tuning on validation data, exceed the budget in $13$--$75\%$ of resampled trials, whereas \ours{} stays within budget in every trial by retaining the reference when the evidence is insufficient. Sequential fastest-first testing with failure-triggered reference evaluation reduces certification from $9{,}000$ to $1{,}997$ rollouts ($77.8\%$ fewer) and from $38.6$ to $7.0$ hours, running the reference only $47$ instead of $1{,}000$ times. On a larger family of $24$ configurations per suite, \ours{} uses $78.9\%$ fewer rollouts than exhaustive evaluation and is the only compared method that certifies acceleration on every suite. The same method certifies reduced flow steps for $\pi_{0.5}$ \citep{intelligence2025pi05} on LIBERO suites, and compute-reduced Qwen3.5-9B \citep{yang2025qwen3} and Llama-3.1-8B \citep{grattafiori2024llama} agents in Crafter \citep{hafner2021benchmarking}. In summary, our key contributions are:
\vspace{-6pt}
\begin{itemize}[left=0pt]

\item We propose \ours{}, a certify-then-select approach that returns the fastest candidate whose acceleration-induced failure risk is certified below a user-specified budget, with a finite-sample guarantee that holds across the candidate family and a lower bound on the fraction of reference-solved episodes preserved.

\item We make certification efficient through two designs tailored to paired, episode-level evaluation: task-balanced sequential testing that stops adaptively and stays valid across tasks of different risk, and failure-triggered reference evaluation that runs the reference only when a candidate fails, without changing any decision.


\item Across OpenVLA-OFT and $\pi_{0.5}$ on LIBERO and language-model agents in Crafter, \ours{} certifies large speedups where common selectors violate the budget, while reducing certification cost by nearly $80\%$.
\end{itemize}
\vspace{-10pt}
\section{Related Work}
\label{sec:related}

\vspace{-6pt}
\paragraph{Efficient inference and adaptive computation.}

Prior work reduces the inference cost of multimodal policies by changing either what is computed or how often computation occurs. VLA acceleration includes action chunking~\citep{kim2025fine}, visual-token pruning~\citep{chen2024image,liu2025bridging,wang2025specprune,pei2026action}, feature caching~\citep{xu2026vla}, and early-exit or dynamic-inference methods~\citep{schuster2022confident,yue2024deer}. Another line lowers deployment cost by transferring knowledge from large or multimodal teachers to
compact students, or by training with auxiliary modalities that are dropped at
inference~\citep{liu2026caml,liu2025mmcd}. For flow-matching policies such as $\pi_0$ and $\pi_{0.5}$, reducing the number of integration steps provides another compute knob~\citep{black2024pi0,intelligence2025pi05}. Related methods also reduce language-model reasoning through dynamic pruning or stopping, or learned allocation strategies~\citep{zheng2026parallel,dai2026small,zheng2026llms, zheng2026dream}. These approaches primarily optimize computation while preserving aggregate performance. They are evaluated by latency
and average task success, which reveal neither which tasks the faster policy newly fails nor whether its
success will hold on new scenes.

\ours{} instead determines whether an accelerated policy has sufficient evidence for deployment under a specified acceleration-induced failure budget. Using only paired terminal outcomes and measured inference cost, it is model agnostic and applies the same certification procedure to action chunking, token pruning, feature caching, flow-step reduction, and their combinations. \ours{} thereby identifies the fastest candidate that can be certified for deployment under a risk budget.

\vspace{-6pt}
\paragraph{Finite-sample risk control.}

Learn-then-Test (LTT) converts candidate selection into multiple hypothesis testing over a finite family~\citep{angelopoulos2025learn,bates2021distribution}, while Pareto Testing uses an estimated candidate ordering to improve selection efficiency~\citep{laufer2022efficiently}. Conformal risk control provides related guarantees under additional loss structure~\citep{angelopoulos2024conformal,angelopoulos2026conformal}. Conformal uncertainty has also been used to modulate how strongly a learner relies on noisy guidance signals~\citep{liu2026adaptive}. Anytime-valid inference, e-values, and test supermartingales permit repeated inspection and early stopping without invalidating the guarantee~\citep{ville1939etude,vovk2021values,ramdas2023gametheoretic,howard2021time}. These tools have been applied to early exits and adaptive computation~\citep{schuster2022confident,jazbec2024fast}. Safe policy improvement instead compares expected return with a baseline~\citep{thomas2015hcpi},
risk-sensitive and distributionally robust methods optimize tail or worst-case
performance~\citep{liu2023data, liu2024towards}, and negative-flip analyses examine inputs on which an
updated model fails after the original succeeds~\citep{yan2021positive}.

Unlike prior work on per-prediction consistency or negative flips over fixed inputs, \ours{} targets closed-loop acceleration risk: the probability that an accelerated policy fails on an episode the reference would solve. Because acceleration changes actions and therefore future states, this risk cannot in general be inferred from marginal success rates, average-success differences, or per-step guarantees; it must be certified on paired episodes. To make this certification computationally tractable while preserving safety guarantees, \ours{} introduces two novel mechanisms: \emph{task-balanced sequential certification} to preserve anytime validity under heterogeneous task risks, and \emph{failure-triggered reference execution}, which avoids running the reference whenever the accelerated policy already succeeds. Together, these mechanisms retain the deployment guarantee while reducing the cost of closed-loop certification.

\vspace{-6pt}
\section{Approach}
\label{sec:app}

\vspace{-6pt}
We present \ours, a training-free approach that decides which accelerator may safely replace a reference policy. \ours{} starts from a computationally heavy reference policy $\pi_0$ and a family of $M$ candidate acceleration configurations $\Lambda=\{\lambda_1,\ldots,\lambda_M\}$. Each configuration $\lambda$, such as an action-chunk length, a visual-token pruning ratio, or a number of flow-integration steps, turns $\pi_0$ into a faster policy $\pi_\lambda$. \ours{} identifies the fastest candidate it can certify as safe on a calibration set, falling back to $\pi_0$ if
no candidate qualifies.

\vspace{-6pt}
\subsection{Problem Formulation}
\label{sec:formulation}

\vspace{-6pt}
Let an evaluation unit $U=(\tau,s,\xi)$ consist of a task index $\tau\in\{1,\ldots,T\}$, an initial scene
$s$ of task $\tau$, and the random seeds $\xi$ used during execution. Running the reference and a candidate on the same unit, that is, from the same scene with the same seeds, is a paired evaluation, and yields binary outcomes $Y_0(U),Y_\lambda(U)\in\{0,1\}$, with $1$ denoting success. Pairing ensures that a
difference between the two outcomes reflects the policies rather than different starting conditions. Units are drawn independently and divided into disjoint sets: a calibration set $\mathcal{D}_{\text{calib}}$ used to certify and select, and a test set $\mathcal{D}_{\text{test}}$ used
to evaluate the selected policy.

 
To evaluate safety, we must isolate failures caused specifically by the acceleration method. Standard marginal metrics (comparing overall success rates) fail to do this: they cannot distinguish an acceleration-induced failure from a task that was simply too difficult for the reference policy to
solve, and they let tasks the accelerated policy happens to repair cancel tasks it breaks. Nor can such failures be measured from individual decisions, since a cheaper decision changes every state that follows (Appendix~\ref{app:identification}). Therefore, we define the acceleration-induced failure (AIF) loss as $L_\lambda(U)=Y_0(U)\bigl(1-Y_\lambda(U)\bigr)\in\{0,1\}$, which equals $1$ if and only if acceleration degrades a reference success into a failure. Deployment risk is then defined under an equal-weight task mixture, so that every task contributes equally: $R(\lambda)=\frac{1}{T}\sum_{t=1}^{T}\mathbb{E}[L_\lambda(U)\mid\tau=t]$.
 
Let $\lambda_0$ denote the reference itself, so $R(\lambda_0)=0$. Given a risk budget $\alpha\in(0,1)$, the largest acceptable AIF risk, and an error level $\delta\in(0,1)$, we seek a rule that maps $\mathcal{D}_{\text{calib}}$ to a configuration $\widehat\lambda\in\Lambda\cup\{\lambda_0\}$, is as fast as possible, and satisfies the PAC-style guarantee~\citep{valiant2026theory}:
\begin{equation}
\Pr\bigl(R(\widehat\lambda)\ge\alpha\bigr)\le\delta,
\label{eq:objective}
\end{equation}
where the probability is over the calibration set and its executions.

We use $R(\lambda)$ as the deployment constraint because it directly measures the probability that acceleration causes a regression on a random deployment episode. We also report the retention $\operatorname{Ret}(\lambda)=\Pr(Y_\lambda{=}1\mid Y_0{=}1)=1-R(\lambda)/\Pr(Y_0{=}1)$, the fraction of reference-solved episodes that acceleration preserves, with expectations under the same task mixture. Hence $R(\lambda)<\alpha$ implies $\operatorname{Ret}(\lambda)>1-\alpha/\Pr(Y_0{=}1)$: the two criteria agree when the reference is reliable but can diverge when it often fails. \ours{} reports a simultaneous lower confidence bound on $\operatorname{Ret}(\widehat\lambda)$ and, when retention is the requirement, can certify a retention budget directly (Appendix~\ref{app:retention_theory}). The two summaries answer complementary questions: AIF risk controls the deployment-wide probability of regression, while retention reports how much of the reference's solved set is preserved.
 

\vspace{-6pt}
\subsection{Certify Then Select}
\label{sec:certification}

\vspace{-6pt}
To achieve the objective, we formulate candidate selection as a statistical hypothesis testing problem following the Learn-then-Test framework~\citep{angelopoulos2025learn}. For each candidate $\lambda_j$, we test the null hypothesis that it is unsafe: $H_j:R(\lambda_j)\ge\alpha$. Candidate $\lambda_j$ is certified for deployment only if $H_j$ is rejected at significance level $\delta_j$.

\vspace{-6pt}
\paragraph{Family-wise error rate control.}
Testing multiple candidates inflates the probability of a false positive, i.e., certifying an unsafe candidate by sheer sampling chance. To prevent this, we bound the family-wise error rate (FWER) by allocating individual error budgets such that $\sum_{j=1}^{M}\delta_j\le\delta$, e.g., $\delta_j=\delta/M$. By the union bound, the overall probability of falsely certifying \emph{any} unsafe candidate across the
entire pool remains bounded by $\delta$. Outside this event every certified candidate is safe, so deploying the fastest certified candidate, or $\lambda_0$ if none is certified, satisfies
Eq.~\ref{eq:objective}.

\vspace{-6pt}
\paragraph{Fixed-sample $p$-values.}
Suppose we evaluate a candidate on a fixed sample of $n$ calibration units. We write
$L_{j,i}=L_{\lambda_j}(U_i)$ and $\widehat R_j=\frac1n\sum_{i=1}^nL_{j,i}$. A small $\widehat R_j$ is evidence against $H_j$, and a $p$-value $\widehat p_j$ quantifies it; $\widehat p_j$ is \emph{valid} if
$\Pr(\widehat p_j\le s)\le s$ for all $s\in[0,1]$ whenever $H_j$ holds. We use the Hoeffding--Bentkus
$p$ value~\citep{bates2021distribution,bentkus2004hoeffding}, which is valid for $[0,1]$-valued losses at
every sample size without distributional assumptions (Eq.~\ref{eq:hb_pvalue},
Appendix~\ref{app:assumptions}), and certify $\lambda_j$ if $\widehat p_j\le\delta_j$.

\vspace{-6pt}
\subsection{Task-Balanced Sequential Certification}
\label{sec:anytime}
 
While the fixed-sample test is statistically valid, evaluating every candidate over the entire calibration set is computationally wasteful. We instead want to halt evaluation the moment a candidate gathers enough evidence to be certified, or accumulates too many failures to ever recover. To this end, we extend \ours{} into an adaptive sequential certification procedure (Algorithm~\ref{alg:care},
Appendix~\ref{app:algorithm}).

\vspace{-6pt}
\paragraph{Task-balanced rounds.}
Benchmarks provide scenes task by task, and tasks differ widely in risk. If units arrived in task order, an early inspection could see only low-risk tasks and certify a candidate that is unsafe on average. We therefore group calibration units into rounds of $m$ units per task, $B=mT$ units in total, and inspect
the evidence only at round boundaries $n_k=kB$, $k=1,\ldots,K$, so every inspection reflects all tasks
in their target proportions.

\vspace{-6pt}
\paragraph{Anytime-valid evidence.}
Recomputing $\widehat p_j$ at every inspection would inflate the error, since each inspection is another chance for a lucky fluctuation. We instead use an $e$-process, whose validity does not depend on when
evaluation stops. With $f_j(n_k)=\sum_{i=1}^{n_k}L_{j,i}$ the number of AIF losses among the first
$n_k$ units,
\begin{equation}
E_j(n_k)=\frac{1}{|\mathcal{Q}|}\sum_{q\in\mathcal{Q}}
\left(\frac{q}{\alpha}\right)^{f_j(n_k)}\left(\frac{1-q}{1-\alpha}\right)^{n_k-f_j(n_k)},
\qquad\mathcal{Q}\subset[0,\alpha).
\label{eq:eprocess}
\end{equation}
Each term is a likelihood ratio comparing the observed losses under a safe risk $q<\alpha$ against the
boundary risk $\alpha$; it grows when losses are rarer than $\alpha$ predicts. Since the true risk is
unknown, we average over a finite grid $\mathcal{Q}$ of safe values. Under $H_j$, task-balanced rounds
make $E_j$ a non-negative supermartingale across inspections, i.e., it does not grow in expectation,
even though tasks differ in risk (Lemma~\ref{lem:round_supermartingale}). By Ville's inequality~\citep{ville1939etude}, the probability that $E_j(n_k)$ ever reaches $1/\delta_j$ at any inspection is bounded by $\delta_j$. Thus, candidate $\lambda_j$ is safely certified the moment
$E_j(n_k)\ge1/\delta_j$, even though this moment depends on the data.

\vspace{-6pt}
\paragraph{Fastest-first search \& early exit.}
To find the fastest safe controller, candidates are pre-sorted from fastest to slowest using the disjoint, small profiling set $\mathcal{D}_{\text{prof}}$ dedicated solely to measuring latency. Algorithm~\ref{alg:care} halts and deploys the first certified candidate it finds. Conversely, if a candidate's early losses are so severe that $E_j(n_K)<1/\delta_j$ even if zero future losses were to
occur ($L_{j,i}=0$ for $i>n_k$), the candidate can never be certified, and an early exit immediately abandons it to save compute (Appendix~\ref{app:validity}).

\vspace{-6pt}
\subsection{Failure-Triggered Reference Evaluation}
\label{sec:lazy}
 
While sequential testing reduces the number of calibration units each candidate needs, executing the heavy reference policy $\pi_0$ during calibration remains a computational bottleneck. \ours{} mitigates
this via a failure-triggered mechanism.
 
Exploiting the structure of the AIF loss $L_{j,i}=Y_0(U_i)\bigl(1-Y_j(U_i)\bigr)$, we observe that if candidate $\lambda_j$ succeeds ($Y_j(U_i)=1$), the loss is $L_{j,i}=0$ regardless of $Y_0(U_i)$, rendering the reference outcome irrelevant. Therefore, the reference $\pi_0$ is executed only when the
candidate fails ($Y_j(U_i)=0$). Furthermore, reference outcomes $Y_0(U_i)$ are stored in a cache $\mathcal{C}$ shared across candidates, so a later failure on the same unit reuses the cached outcome. As a result, $\pi_0$ is executed at most once per unit, and only on units where some evaluated candidate
fails, bounding the total number of reference executions by $N_{\text{ref}}=|\mathcal{C}|\le\min\{n_K,\sum_jF_j\}$, where $F_j$ is the number of failures of $\lambda_j$, while yielding exactly the same losses, and hence the same certification decisions, as
exhaustive paired evaluation (Lemma~\ref{lem:failure_triggered}, Appendix~\ref{app:validity}).
 
Combining fastest-first sequential certification with failure-triggered reference evaluation gives the complete \ours{} controller. Candidate configurations are considered in increasing order of measured
inference cost, evaluated on task-balanced rounds, and certified with anytime-valid evidence. The accelerated policy induced by the first certified configuration is deployed; if no candidate can be certified, the reference policy $\pi_0$ is retained. The candidate ordering is fixed in advance on
$\mathcal{D}_{\text{prof}}$, and early exit and failure-triggered reference evaluation affect only calibration cost, not which candidate is certified, so the finite-sample guarantee in Eq.~\ref{eq:objective} holds for the deployed configuration (Theorem~\ref{thm:risk_control}).

\vspace{-6pt}
\section{Experiments}
\label{sec:exp}
 
\vspace{-5pt}
 
We evaluate \ours{} across different benchmarks, LIBERO~\citep{liu2023libero} and Crafter~\citep{hafner2021benchmarking}, spanning multiple model families, including
OpenVLA-OFT~\citep{kim2025fine}, the flow-matching policy
$\pi_{0.5}$~\citep{intelligence2025pi05}, and the language-model agents Qwen3.5-9B~\citep{yang2025qwen3} and
Llama-3.1-8B~\citep{grattafiori2024llama}. Our experiments address four research questions: \textbf{Q1:} Does \ours{} select useful
accelerated controllers while maintaining family-level AIF guarantees? \textbf{Q2:} How does \ours{}
compare against baseline selectors in maintaining safe risk control? \textbf{Q3:} Do sequential and
reference-query rules effectively reduce certification cost? \textbf{Q4:} Does \ours{} generalize across
diverse model architectures and tasks?
 
\vspace{-6pt}
\subsection{Certified OpenVLA Deployment}
\label{sec:primary_deployment}
 
We first evaluate OpenVLA-OFT on the four LIBERO suites (Figure \ref{fig:libero_tasks}, Appendix \ref{app:benchmarks}), Goal, Spatial, Object, and Long, each with ten
tasks, using one-step replanning (H1) as the reference. The candidate family contains eight
configurations, all with action horizon 8: plain execution, FastV-25, FastV-50,
FastV-75~\citep{chen2024image}, VLA-Cache~\citep{xu2026vla}, VLA-Pruner~\citep{liu2025bridging},
SpecPrune~\citep{wang2025specprune}, and VLA-ADP~\citep{pei2026action}. We
evaluate 500 matched scenes per suite, for 2,000 initial scenes and 18,000 rollouts across the reference
and candidates. Certification uses 250 task-balanced episodes per suite; the remaining 250 episodes form the test set.
 
\begin{wraptable}{r}{0.5\textwidth}
\vspace{-20pt}
\caption{\textbf{Certified CARE deployments.} Speedup is relative to H1; preserved is the simultaneous
95\% lower bound on the fraction of H1-solved episodes retained.}
\label{tab:primary_deployment}
\vspace{2pt}
\small
\begin{tabular}{lccc}
\toprule
Suite & Selected policy & Speedup & Preserved \\
\midrule
Goal    & FastV-25 H8 & $9.02\times$  & $93.5\%$ \\
Spatial & FastV-75 H8 & $10.08\times$ & $95.7\%$ \\
Object  & Plain H8    & $9.56\times$  & $92.2\%$ \\
Long    & VLA-Cache H8 & $10.81\times$ & $85.8\%$ \\
\bottomrule
\end{tabular}
\vspace{-10pt}
\end{wraptable}
 
For LIBERO, each evaluation unit fixes the task, initial scene, and random seed before either policy is
executed. We set the AIF risk budget to $\alpha=5.0\%$ and the error probability to $\delta=5.0\%$; see a sensitivity analysis in  Appendix~\ref{app:sensitivity}. On the test set, we report two metrics: (1) speedup over reference, the ratio of the reference policy's mean
inference time to that of the deployed controller (larger is faster); and (2) preserved success, a
family-corrected $95\%$ lower confidence bound on the fraction of reference-solved episodes preserved under
acceleration. The family-wise correction keeps the reported bounds valid after selecting among multiple
candidates across all four suites.

Table~\ref{tab:primary_deployment} summarizes the certification outcomes. \ours{} certifies an
accelerated policy on every suite, achieving $9.02$--$10.81\times$ speedups, while the simultaneous lower
bounds guarantee that the selected policies preserve $85.8\%$--$95.7\%$ of reference-solved episodes.
These results answer \textbf{Q1}: \ours{} selects useful accelerated policies with family-level
guarantees.
 
 
\vspace{-6pt}
\subsection{Comparison with Deployment Selectors}
\label{sec:baselines}
 
\vspace{-6pt}
We compare \ours{} with alternative selectors through controlled resampling. For each LIBERO suite, the
500 stored paired episodes define an empirical population. In each of 10,000 trials, we sample 250
calibration episodes, apply each method, and evaluate the selected policy against the full population. We
organize the comparison around two questions. First, under tight risk budgets, do commonly used selectors
respect the requested risk? Second, at the standard budget, when all risk-controlled methods are valid,
which method selects faster policies and requires fewer evaluation rollouts?

\vspace{-6pt}
\paragraph{Selection validity under tighter budgets.}
We first examine whether each method respects the requested risk budget. Fixed-$n$ \ours{} evaluates
every candidate on all 250 calibration episodes and applies a family-corrected certificate. Latency only
always selects the fastest policy. Average success selects the fastest policy whose empirical success is
within one percentage point of H1. Validation tuned selects the fastest policy whose observed AIF rate is
below $\alpha$, without correcting for selection over multiple candidates. The unpaired test compares
marginal success rates rather than paired AIF outcomes. We report the over-budget rate, the percentage of
trials in which the selected policy's population AIF exceeds $\alpha$.
 
As shown in Table~\ref{tab:tight_budget_validity}, \ours{} maintains validity under tighter budgets by
safely retaining H1 when no accelerated policy has sufficient evidence for certification; its validity
thus reflects safe fallback rather than forced acceleration. In contrast, latency-only and
average-success selection exceed the budget in $75.0\%$ and $25.0\%$ of trials at $\alpha=2.0\%$ and
$2.5\%$, respectively, and validation tuning and the unpaired test also frequently violate it. Their
apparent speedups arise because they almost always deploy an accelerated policy despite insufficient
evidence. These results directly answer \textbf{Q2}, confirming that \ours{} delivers safe risk control
where common selectors do not.

\begin{wraptable}{r}{0.5\textwidth}
\vspace{-22pt}
\centering
\caption{\textbf{Over-budget rate (\%) under tighter AIF budgets.} \ours{} has no observed over-budget deployments, whereas the other selectors exceed the budget in $13.0$--$75.0\%$ of trials.}
\label{tab:tight_budget_validity}

\vspace{2pt}
\small
\setlength{\tabcolsep}{10pt}
 
\begin{tabular}{lcc}
\toprule
Method
& {$\alpha=2.0\%$}
& {$\alpha=2.5\%$} \\
\midrule
Latency only       & $75.0$          & $25.0$ \\
Average success    & $75.0$          & $25.0$ \\
Validation tuned   & $36.8$          & $13.0$ \\
Unpaired test      & $67.1$          & $19.5$ \\
\ours  & $\mathbf{0.0}$ & $\mathbf{0.0}$ \\
\bottomrule
\end{tabular}
\vspace{-10pt}
\end{wraptable}

\vspace{-6pt}
\paragraph{Deployment speed and evaluation cost.}
We next compare selectors at the standard budget $\alpha=5.0\%$, reporting mean speedup relative to H1,
fallback rate, and mean evaluation rollouts. Pareto Testing splits episodes to rank candidates by
estimated risk before testing, while cost-order LTT tests policies from fastest to slowest; both ordered
methods stop at the first uncertifiable policy. Safe policy improvement controls net success relative to
H1, allowing newly solved episodes to offset failures, and therefore does not directly control AIF.
Sequential \ours{} tests paired AIF with interim evidence checks, stops early when certification is
achieved or unattainable, and runs H1 only after an accelerated policy fails.
 
As shown in Table~\ref{tab:selector_validity}, fixed-$n$ \ours{} achieves the highest mean speedup among
risk-controlling selectors ($8.88\times$ with $11.7\%$ fallback). In contrast, Pareto Testing
($6.16\times$, $42.2\%$ fallback) and cost-order LTT ($6.35\times$, $45.5\%$ fallback) fall back more often
because their rigid evaluation order can halt before reaching a certifiable policy, and safe policy
improvement yields $5.30\times$ with $60.5\%$ fallback. Crucially, sequential \ours{} maintains a
competitive $7.32\times$ speedup ($26.9\%$ fallback) while reducing evaluation cost from 2,250 to just 592
rollouts. These results directly answer \textbf{Q3}, confirming that sequential and reference-query rules
substantially reduce certification cost.
 
\vspace{-10pt}
\begin{table}[ht]
\centering
 
\caption{\textbf{Deployment speed and evaluation cost.}
Fixed-$n$ \ours{} achieves the highest mean speedup; sequential \ours{} reduces evaluation cost by
$73.7\%$ while keeping a $7.32\times$ speedup and less fallback than all baselines.}
\label{tab:selector_validity}
\vspace{3pt}
\small
\begin{tabular}{lrrr}
\toprule
Method
& \shortstack{Mean speedup ($\times$)}
& \shortstack{Fallback (\%)}
& \shortstack{Mean rollouts} \\
\midrule
\ours, fixed-$n$              & $\mathbf{8.88}$ & $11.7$ & $2{,}250$ \\
Pareto Testing                & $6.16$          & $42.2$ & $2{,}250$ \\
Cost-order fixed-sequence LTT & $6.35$          & $45.5$ & $2{,}250$ \\
Safe policy improvement       & $5.30$          & $60.5$ & $2{,}250$ \\
\textbf{\ours, sequential}    & $7.32$          & $26.9$ & $\mathbf{592}$ \\
\bottomrule
\end{tabular}
\vspace{-5pt}
\end{table}

\vspace{-6pt}
\subsection{Evaluation with a Broader Candidate Family}
\label{sec:broader_family}
 
\vspace{-6pt}
We expand the candidate family to 24 OpenVLA-OFT configurations formed by four action horizons,
$\mathrm{H}\in\{2,4,6,8\}$, and six inference modes: plain inference, FastV-25, FastV-50, FastV-75,
VLA-Cache, and VLA-Pruner. We use fresh calibration streams while retaining the same H1 reference, AIF
budget, and family-level risk control. This experiment tests whether \ours{} remains effective and
evaluation-efficient when certification must search a larger and more heterogeneous policy family.
 
\vspace{-6pt}
\paragraph{Baselines and metrics.}
Full-sample Hoeffding--Bentkus (HB) evaluates all 24 candidates and H1 on all 250 calibration episodes.
Sequential HB and sequential exact inspect the evidence at five pre-specified points, using
Hoeffding--Bentkus and exact-binomial tests, respectively. The plug-in e-process is an alternative
anytime-valid sequential test. \ours{} combines task-balanced sequential testing with fastest-first
evaluation, early abandonment of candidates that can no longer certify, and failure-triggered H1
evaluation. For each method, we report the selected acceleration mechanism and horizon per suite, and
measure evaluation cost by candidate rollouts, H1 rollouts, their total, and the reduction relative to
full-sample HB.
 
\begin{table}[ht]
\centering
\caption{\textbf{Evaluation with a broader acceleration candidate family.}
Reference denotes fallback to H1. \ours{} is the only method that certifies an accelerated policy on every
suite, using 5,263 rollouts, $78.9\%$ fewer than full-sample HB and fewer than every sequential
baseline. Rollouts are aggregated over the four LIBERO suites.}
\label{tab:main_care}
\vspace{3pt}
 
\resizebox{0.92\textwidth}{!}{
\begin{tabular}{lllcrrrr}
\toprule
Method &
Suite &
Acceleration &
Horizon &
\shortstack{Candidate\\rollouts} &
\shortstack{Reference\\rollouts} &
\shortstack{Total\\rollouts} &
\shortstack{Rollout\\saving (\%)} \\
\midrule
 
\multirow{4}{*}{Full-sample HB}
& Goal    & FastV-50   & H8
& \multirow{4}{*}{$24{,}000$}
& \multirow{4}{*}{$1{,}000$}
& \multirow{4}{*}{$25{,}000$}
& \multirow{4}{*}{$0.0$} \\
& Spatial & VLA-Cache  & H2 & & & & \\
& Object  & VLA-Pruner & H8 & & & & \\
& Long    & Reference  & -- & & & & \\
\midrule
 
\multirow{4}{*}{Sequential HB}
& Goal    & Reference  & --
& \multirow{4}{*}{$7{,}150$}
& \multirow{4}{*}{$156$}
& \multirow{4}{*}{$7{,}306$}
& \multirow{4}{*}{$70.8$} \\
& Spatial & Reference  & -- & & & & \\
& Object  & VLA-Pruner & H8 & & & & \\
& Long    & Reference  & -- & & & & \\
\midrule
 
\multirow{4}{*}{Sequential exact}
& Goal    & FastV-50   & H8
& \multirow{4}{*}{$6{,}350$}
& \multirow{4}{*}{$135$}
& \multirow{4}{*}{$6{,}485$}
& \multirow{4}{*}{$74.1$} \\
& Spatial & VLA-Cache  & H2 & & & & \\
& Object  & VLA-Pruner & H8 & & & & \\
& Long    & Reference  & -- & & & & \\
\midrule
 
\multirow{4}{*}{Plug-in e-process}
& Goal    & FastV-25   & H8
& \multirow{4}{*}{$5{,}750$}
& \multirow{4}{*}{$123$}
& \multirow{4}{*}{$5{,}873$}
& \multirow{4}{*}{$76.5$} \\
& Spatial & Reference  & -- & & & & \\
& Object  & VLA-Pruner & H8 & & & & \\
& Long    & Reference  & -- & & & & \\
\midrule
 
\multirow{4}{*}{\textbf{\ours}}
& Goal    & \textbf{FastV-25}     & \textbf{H8}
& \multirow{4}{*}{$\mathbf{5{,}150}$}
& \multirow{4}{*}{$\mathbf{113}$}
& \multirow{4}{*}{$\mathbf{5{,}263}$}
& \multirow{4}{*}{$\mathbf{78.9}$} \\
& Spatial & \textbf{FastV-75}  & \textbf{H8} & & & & \\
& Object  & \textbf{Plain}     & \textbf{H8} & & & & \\
& Long    & \textbf{VLA-Cache} & \textbf{H8} & & & & \\
\bottomrule
\end{tabular}
}
\vspace{-8pt}
\end{table}
 
Table~\ref{tab:main_care} shows that expanding the candidate family does not diminish the effectiveness
or efficiency of \ours{}, which certifies an accelerated policy on all four suites. In contrast, full-sample HB and
sequential exact retain H1 on Long, the plug-in e-process retains H1 on Spatial and Long, and sequential
HB certifies acceleration only on Object. \ours{} also has the lowest evaluation cost, using 5,263 rather
than 25,000 rollouts ($78.9\%$ fewer) and reducing H1 evaluation from 1,000 to 113 rollouts. Compared with
the sequential baselines, it uses 610, 1,222, and 2,043 fewer rollouts than the plug-in e-process,
sequential exact, and sequential HB, respectively. Thus, over the broader family, \ours{} provides both the
widest certification coverage and the largest rollout saving.
 
\vspace{-6pt}
\subsection{Cross-Architecture and Cross-Domain Evaluation}
\label{sec:generalization}
 
\ours{} requires only paired terminal outcomes and measured inference cost; it does not depend on a
particular policy architecture or acceleration mechanism. We evaluate this method on the flow-matching
VLA policies $\pi_0$ and $\pi_{0.5}$,  and on language-model agents with
Qwen3.5-9B and Llama-3.1-8B in Crafter (Figure \ref{fig:crafter_tasks}, Appendix \ref{app:benchmarks}).


\vspace{-6pt}
\paragraph{Flow-step selection for $\pi_{0.5}$.}
Beyond OpenVLA-OFT, we evaluate $\pi_{0.5}$ and use its default ten-step inference with action horizon H8
as the reference. The candidate family contains the same policy with 2, 4, or 6 flow-refinement
steps, ordered from lowest to highest computation. For each LIBERO suite, we use 250 task-stratified
episodes, examine the evidence every 50 episodes, and control the three-candidate family at $\alpha=0.05$.
Exhaustive evaluation would require 1,000 rollouts per suite.
 
Table~\ref{tab:pi05_family} shows that \ours{} generalizes to a different VLA architecture. It selects the
two-step policy on Spatial and Object, reducing flow computation by $5\times$, and the four-step policy on
Goal, reducing it by $2.5\times$. On Long, \ours{} stops all three candidates early and retains the
reference, requiring 464 rather than 1,000 rollouts. Across four suites, it uses 1,250 candidate
rollouts and only 26 reference rollouts owing to failure-triggered evaluation, $68.1\%$ fewer than
exhaustive evaluation.
 
\vspace{-8pt}
\begin{table*}[ht]
\centering
\caption{\textbf{Certified flow-step selection for $\pi_{0.5}$ on LIBERO.}
\ours{} certifies reduced-step policies on Goal, Spatial, and Object and retains the ten-step reference on
Long, using 1,276 rather than 4,000 rollouts ($68.1\%$ fewer). Compute and speedup are relative to the
ten-step reference.}
\label{tab:pi05_family}
\vspace{3pt}
\small
\resizebox{0.8\textwidth}{!}{
\begin{tabular}{llcccc}
\toprule
Suite &
Selected policy &
\shortstack{Compute\\(\% ref.)} &
\shortstack{Speedup\\($\times$)} &
\shortstack{Total\\rollouts} &
\shortstack{Rollout\\saving (\%)} \\
\midrule
Goal    & Four-step  & 40  & 1.26 & 305 & 69.5 \\
Spatial & Two-step   & 20  & 1.45 & 254 & 74.6 \\
Object  & Two-step  & 20  & 1.45 & 253 & 74.7 \\
Long    & Reference        & 100 & 1.00 & 464 & 53.6 \\
\midrule
\multicolumn{4}{r}{\textbf{Total}} &
$\mathbf{1{,}276}$ &
$\mathbf{68.1}$ \\
\bottomrule
\end{tabular}
}
\vspace{-15pt}
\end{table*}

 
\paragraph{Compute-reduced language-model agents in Crafter.}
Next, we evaluate whether \ours{} extends beyond robotic policies. Qwen3.5-9B and Llama-3.1-8B agents
act in Crafter using learned schedules that reduce language-model computation on the \texttt{collect\_wood} and \texttt{collect\_drink} tasks. Each schedule is trained before certification
and evaluated on a reserved stream. As shown in Table~\ref{tab:crafter_main}, both models certify their
\texttt{collect\_wood} schedules on every seed while reducing compute by 25.1\%--30.6\%, and the
\texttt{collect\_drink} schedules certify on two of three seeds with smaller reductions. Together, the
$\pi_{0.5}$ and Crafter results directly answer \textbf{Q4}, confirming that \ours{} supports
different policy architectures, compute controls, benchmarks, and task domains.

\vspace{-6pt}
\subsection{Design Analysis and Ablations}
\label{sec:ablation}

\vspace{-6pt}
We analyze five components of \ours{} in two groups. The validity analysis evaluates (1) paired
episode-level outcomes, (2) family correction when selecting among multiple candidates, and (3)
task-balanced acquisition at sequential looks. The efficiency analysis isolates (4) sequential
fastest-first candidate evaluation and (5) failure-triggered reference evaluation. Complete settings and
results are in Appendix~\ref{app:design_analysis}.

 \begin{wraptable}{r}{0.5\textwidth}
\vspace{-20pt}
\centering
\caption{\textbf{Certified compute reduction for Crafter agents.}
\texttt{collect\_wood} schedules certify on all seeds at 69.4\%--74.9\% of reference compute;
\texttt{collect\_drink} schedules certify on two of three with compute reduced.}
\label{tab:crafter_main}
\vspace{2pt}
\setlength{\tabcolsep}{5.5pt}
\begin{tabular}{llcc}
\toprule
Model &
Task &
\shortstack{Certified\\seeds} &
\shortstack{Compute\\(\% ref.)} \\
\midrule
Qwen3.5-9B   & Wood  & $3/3$ & 69.4 \\
Qwen3.5-9B   & Drink & $2/3$ & 85.8 \\
Llama-3.1-8B & Wood  & $3/3$ & 74.9 \\
Llama-3.1-8B & Drink & $2/3$ & 98.6 \\
\bottomrule
\end{tabular}
\vspace{-5pt}
\end{wraptable}
 
\vspace{-6pt}
\paragraph{Validity analysis.}
Local acceleration effects do not consistently predict whole-episode AIF. For example, the local and
episode-level rates are $4.3\%$ and $1.0\%$ on Object, compared with $9.5\%$ and $2.7\%$ on Long
(Table~\ref{tab:local_episode}), which supports certifying complete paired episodes rather than individual
decisions. Family correction is also necessary (Table~\ref{tab:multiplicity}): in a controlled
ten-candidate simulation, uncorrected selection deploys an over-budget policy in $46.0\%$ of trials,
whereas family-corrected testing remains at or below $0.15\%$. Finally, balanced round-robin acquisition
(Table~\ref{tab:balanced_order}) stays within the candidate-level error budget for every tested
heterogeneous task-risk profile, whereas task-by-task acquisition can certify an over-budget policy before
its high-risk tasks are observed.
 
\vspace{-6pt}
\paragraph{Efficiency analysis.}
Sequential fastest-first evaluation reduces candidate rollouts from $8{,}000$ to $1{,}950$ by stopping
candidates that can no longer be certified. Failure-triggered evaluation further reduces H1 rollouts from
$1{,}000$ to 47 by running the reference only after candidate failures. Together, the two components reduce
total evaluation from $9{,}000$ to $1{,}997$ rollouts, a $77.8\%$ reduction, and wall-clock time from
$38.6$ to $7.0$ hours, an $81.9\%$ reduction (Table~\ref{tab:cost_ablation}).

 

\vspace{-6pt}

\section{Conclusions}
\label{sec:conclusion}

\vspace{-6pt}
We introduce \ours{}, an approach for certified accelerator selection based on acceleration-induced failure, a paired, episode-level risk that cannot in general be identified from average success or per-step measurements alone. \ours{} returns the fastest candidate whose risk can be certified below a user-specified budget and retains the reference when the available evidence is insufficient. Task-balanced sequential testing and failure-triggered reference evaluation substantially reduce the closed-loop rollout cost of certification while preserving the deployment guarantee. Across OpenVLA-OFT and $\pi_{0.5}$ on LIBERO and language-model agents in Crafter, \ours{} applies across different policy classes and acceleration mechanisms, recovers substantial acceleration when supported by the data, and avoids the over-budget selections produced by common heuristic deployment rules. These results suggest that acceleration need not be treated only as a speed--accuracy trade-off: it can instead be selected with explicit, finite-sample control over the failures it introduces.


\bibliography{ref}
\bibliographystyle{iclr2027_conference}

\newpage
\appendix

\section{Theoretical Details}
\label{app:additional_theory}
 
This appendix gives the full procedure (App.~\ref{app:algorithm}), the assumptions and the $p$-value
(App.~\ref{app:assumptions}), the validity proof (App.~\ref{app:validity}), the supermartingale property
it relies on (App.~\ref{app:supermartingale}), the identification results
(App.~\ref{app:identification}), and the retention bounds (App.~\ref{app:retention_theory}).
 
\subsection{The Sequential Procedure}
\label{app:algorithm}
 
Algorithm~\ref{alg:care} summarizes \ours{}. Candidates are evaluated in fastest-first order, each on the
calibration units round by round. The reference is executed on a unit only when the candidate fails there
and the reference outcome is not yet cached (lines 8--15). At each inspection point, the candidate is
certified if its $e$-process reaches $1/\delta_j$ (lines 18--21), and abandoned if it can no longer reach
this threshold even with no further losses (lines 22--25). The fixed-sample version is the special case
with a single inspection at $n_K=n$, where the $e$-process test is replaced by $\widehat p_j\le\delta_j$.
 
\begin{algorithm}[ht]
\caption{\ours: certify-then-select with task-balanced rounds and failure-triggered reference evaluation}
\label{alg:care}
\small
\begin{algorithmic}[1]
\Require Reference policy $\pi_0$; candidates $\Lambda=\{\lambda_1,\ldots,\lambda_M\}$; risk budget $\alpha$;
error levels $\{\delta_j\}_{j=1}^{M}$ with $\sum_j\delta_j\le\delta$; grid $\mathcal{Q}\subset[0,\alpha)$;
profiling set $\mathcal{D}_{\text{prof}}$; calibration units $U_1,\ldots,U_{n_K}$ arranged in $K$ rounds of
$m$ units per task, with inspection points $n_k=kB$, $B=mT$
\Ensure Selected configuration $\widehat\lambda\in\Lambda\cup\{\lambda_0\}$
\State Measure latencies on $\mathcal{D}_{\text{prof}}$ and index $\Lambda$ from fastest ($\lambda_1$) to slowest ($\lambda_M$)
\State Initialize the reference-outcome cache $\mathcal{C}\gets\emptyset$
\For{$j=1,\ldots,M$}
    \State $f_j\gets0$ \Comment{number of AIF losses of $\lambda_j$ so far}
    \For{$k=1,\ldots,K$}
        \For{$i=n_{k-1}+1,\ldots,n_k$}
            \State Run $\pi_{\lambda_j}$ on $U_i$ and observe $Y_j(U_i)$
            \If{$Y_j(U_i)=1$}
                \State $L_{j,i}\gets 0$
                \Comment{candidate succeeds: loss is $0$ whatever $\pi_0$ does}
            \Else
                \If{$U_i$ is not in the reference cache $\mathcal C$}
                    \State Run $\pi_0$ on $U_i$ and store
                    $\mathcal C[U_i]\gets Y_0(U_i)$
                    \Comment{failure-triggered}
                \EndIf
                \State $L_{j,i}\gets \mathcal C[U_i]$
                \Comment{candidate fails: use cached reference outcome}
            \EndIf
            \State $f_j\gets f_j+L_{j,i}$
        \EndFor
        \State Compute $E_j(n_k)$ from $f_j$ by Eq.~\ref{eq:eprocess}
        \If{$E_j(n_k)\ge1/\delta_j$}
            \State \Return $\widehat\lambda=\lambda_j$ \Comment{certified; fastest certified candidate}
        \EndIf
        \State $E_j^{\max}\gets\frac{1}{|\mathcal{Q}|}\sum_{q\in\mathcal{Q}}\bigl(\frac{q}{\alpha}\bigr)^{f_j}\bigl(\frac{1-q}{1-\alpha}\bigr)^{n_K-f_j}$ \Comment{no further losses}
        \If{$E_j^{\max}<1/\delta_j$}
            \State \textbf{break} \Comment{early exit: $\lambda_j$ can no longer be certified}
        \EndIf
    \EndFor
\EndFor
\State \Return $\widehat\lambda=\lambda_0$ \Comment{no candidate certified: keep the reference}
\end{algorithmic}
\end{algorithm}

\subsection{Assumptions and the Hoeffding--Bentkus $p$-Value}
\label{app:assumptions}
 
We write $R_t(\lambda)=\mathbb{E}[L_\lambda(U)\mid\tau=t]$ for the per-task risk, so that $R(\lambda)=\frac1T\sum_tR_t(\lambda)$. The guarantees rest on three assumptions.
\begin{itemize}
\item[(A1)] \emph{Pre-specification.} The candidates, $\alpha$, $\delta$ and its split $\{\delta_j\}$, the round schedule, the grid $\mathcal Q$, and the fastest-first order (measured on $\mathcal{D}_{\text{prof}}$) are fixed before any calibration outcome is observed.
\item[(A2)] \emph{Sampling.} Calibration units are independent, and scenes are drawn i.i.d.\ within each
task from the deployment distribution. Fixed-sample units are either i.i.d.\ from the equal-weight task
mixture or stratified with $n/T$ units per task; sequential rounds contain exactly $m$ units of every task.
\item[(A3)] \emph{Executions.} Each execution launched from a unit depends only on that unit and on
randomness independent of all other executions. Residual simulator nondeterminism is allowed and is part
of the law of $(Y_0(U),Y_\lambda(U))$ that defines $R$.
\end{itemize}

 
\paragraph{The Hoeffding--Bentkus $p$-value.}
For independent $[0,1]$-valued $Z_1,\ldots,Z_n$ with empirical mean $\widehat Z$ and a threshold
$\theta\in(0,1)$,
\begin{equation}
\widehat p(\widehat Z;\theta)=\min\Bigl\{1,\;e\,\Pr\bigl[\operatorname{Bin}(n,\theta)\le\lceil n\widehat Z\rceil\bigr],\;
\exp\bigl(-n\,\mathrm{KL}(\widehat Z\wedge\theta\,\|\,\theta)\bigr)\Bigr\},
\label{eq:hb_pvalue}
\end{equation}
where $a\wedge b=\min\{a,b\}$, $\mathrm{KL}(a\|b)=a\log\frac ab+(1-a)\log\frac{1-a}{1-b}$, and $e$ is
Euler's number. If the average mean $\frac1n\sum_i\mathbb{E}[Z_i]$ is at least $\theta$, the binomial term
(Bentkus's inequality~\citep{bentkus2004hoeffding}) and the exponential term (Hoeffding's
inequality~\citep{hoeffding1963probability}) both upper bound the distribution function
$F(z)=\Pr(\widehat Z\le z)$ at every $z$; neither inequality requires identically distributed variables.
Hence $\widehat p\ge F(\widehat Z)$, and $\Pr(\widehat p\le s)\le s$ for all
$s$~\citep{bates2021distribution}. \ours{} uses $\widehat p_j=\widehat p(\widehat R_j;\alpha)$ with
$Z_i=L_{j,i}$. Under (A2) the means of the $L_{j,i}$ average to $R(\lambda_j)$, both for i.i.d.\ and for
stratified sampling, so $\widehat p_j$ is valid under $H_j$.
 
\subsection{Validity of \ours{}}
\label{app:validity}
 
The proofs use one device. By (A3), all execution randomness can be drawn before the procedure starts,
which fixes an \emph{outcome table} with entries $Y_0(U_i)$ and $Y_j(U_i)$ for every calibration unit and
every candidate. The procedure only decides which entries to reveal; revealing an entry later, or not at
all, does not change it, and revealed entries have the same joint law as the outcomes of an actual run. We
define $L_{j,i}$ and $E_j(n_k)$ from the full table, for every $j$ and $k\le K$.
 
\begin{lemma}[Failure-triggered reference evaluation]
\label{lem:failure_triggered}
Under (A3), on the outcome table: (i) every loss computed with failure-triggered evaluation equals
$L_{j,i}=Y_0(U_i)(1-Y_j(U_i))$, so the evaluated units, all certification decisions, and $\widehat\lambda$
coincide with those of exhaustive paired evaluation; (ii) the reference is executed at most once per unit
and only on units where some evaluated candidate fails, so
$N_{\mathrm{ref}}=|\mathcal{C}|\le\min\{n_K,\sum_jF_j\}$, where $F_j$ is the number of failures of
$\lambda_j$ on the units it is evaluated on.
\end{lemma}
 

\begin{proof}
(i) If $Y_j(U_i)=1$, the procedure assigns $L_{j,i}=0$, which is correct for either value of $Y_0(U_i)$. If $Y_j(U_i)=0$, the procedure uses the cached reference outcome $\mathcal{C}[U_i]$. If $U_i$ has not yet been cached, $\pi_0$ is executed once on $U_i$ and the result is stored as $\mathcal{C}[U_i]=Y_0(U_i)$. Hence, in either case, $L_{j,i}=\mathcal{C}[U_i]=Y_0(U_i)=Y_0(U_i)(1-Y_j(U_i))$, so every computed loss is exactly the same as under exhaustive paired evaluation. Since the next unit evaluated and all subsequent stopping and certification decisions depend only on the losses observed so far, the failure-triggered procedure and exhaustive paired evaluation make the same decisions by induction.

(ii) The reference is executed on $U_i$ only after a candidate fails on that unit and only if $U_i$ is not already present in the reference cache $\mathcal{C}$. Once executed, its outcome is cached and reused for all later candidates. Therefore, each unit triggers at most one reference execution, and every reference execution corresponds to a distinct candidate failure. Thus, $N_{\mathrm{ref}}=|\mathcal{C}|\le \min\{n_K,\sum_j F_j\}$.
\end{proof}

\begin{theorem}[Risk control]
\label{thm:risk_control}
Under (A1)--(A3), \ours{}, in its fixed-sample and sequential versions and with fastest-first search,
early exit, and failure-triggered reference evaluation, satisfies
$\Pr\bigl(R(\widehat\lambda)\ge\alpha\bigr)\le\delta$.
\end{theorem}
 
\begin{proof}
By Lemma~\ref{lem:failure_triggered}(i), we may analyze exhaustive paired evaluation on the outcome table.
 
\emph{Each test has level $\delta_j$.} In the fixed-sample version, $\widehat p_j$ is valid
(App.~\ref{app:assumptions}), so $\Pr_{H_j}(\widehat p_j\le\delta_j)\le\delta_j$. In the sequential version,
$(E_j(n_k))_{k=0}^K$ is a non-negative supermartingale under $H_j$ with $E_j(0)=1$
(Lemma~\ref{lem:round_supermartingale}), so Ville's inequality~\citep{ville1939etude} gives
$\Pr_{H_j}(\max_{k\le K}E_j(n_k)\ge1/\delta_j)\le\delta_j$. A candidate is certified only at an inspection
where $E_j(n_k)\ge1/\delta_j$, so this bound holds whatever data-dependent inspection evaluation stops at.
 
\emph{Early exit.} For $q\in\mathcal{Q}$, write $a_q=q/\alpha<1<b_q=(1-q)/(1-\alpha)$. The $q$-term of
Eq.~\ref{eq:eprocess} after $n$ units with $f$ losses is $a_q^fb_q^{n-f}$, which decreases in $f$ and
increases in $n$. Since losses only accumulate, every later value of $E_j$ is at most $E_j^{\max}$, the
value with no further losses at $n_K$ (Algorithm~\ref{alg:care}, line 22). If $E_j^{\max}<1/\delta_j$,
the candidate would never be certified, so abandoning it does not change which candidates are certified.
 
\emph{Selection.} By the union bound, the probability that some unsafe candidate
($R(\lambda_j)\ge\alpha$) is certified is at most $\sum_j\delta_j\le\delta$. Otherwise, every certified
candidate, in particular the deployed one, has $R<\alpha$, and the fallback has $R(\lambda_0)=0$.
\end{proof}
 
\subsection{Supermartingale Property under Task-Balanced Rounds}
\label{app:supermartingale}
 
Let $\mathcal{F}_k$ be generated by the outcome-table rows of rounds $1,\ldots,k$, and for $q\in\mathcal{Q}$
let $M_{j,i}(q)=(q/\alpha)^{L_{j,i}}\bigl((1-q)/(1-\alpha)\bigr)^{1-L_{j,i}}$, so that the $q$-term of
Eq.~\ref{eq:eprocess} is $E_j(n_k,q)=\prod_{i\le n_k}M_{j,i}(q)$.
 
\begin{lemma}
\label{lem:round_supermartingale}
Under (A2), (A3), and $H_j:R(\lambda_j)\ge\alpha$, $(E_j(n_k))_{k\ge0}$ is a non-negative supermartingale
with respect to $(\mathcal{F}_k)$.
\end{lemma}
 
\begin{proof}
Fix $q$ and round $k$. Its units are fresh draws, independent of $\mathcal{F}_{k-1}$ and of each other, and
unit $i$ of task $t(i)$ has conditional loss mean $\mu_i=R_{t(i)}(\lambda_j)$. Since each round contains
$m$ units of every task, $\frac1B\sum_i\mu_i=R(\lambda_j)\ge\alpha$. With $c=\frac{q-\alpha}{\alpha(1-\alpha)}<0$,
$\mathbb{E}[M_{j,i}(q)\mid\mathcal{F}_{k-1}]\le1+c(\mu_i-\alpha)$, with equality for binary losses; this holds for any $[0,1]$-valued loss because $a^zb^{1-z}\le za+(1-z)b$ for $z\in[0,1]$, and the right-handside is non-negative. By conditional independence and the inequality of arithmetic and geometric means,
\begin{equation}
\mathbb{E}\Bigl[\prod_{i\in\text{round }k}M_{j,i}(q)\Bigm|\mathcal{F}_{k-1}\Bigr]
=\prod_{i\in\text{round }k}\bigl(1+c(\mu_i-\alpha)\bigr)
\le\bigl(1+c(R(\lambda_j)-\alpha)\bigr)^{B}\le1 .
\end{equation}
Hence each $E_j(n_k,q)$ is a non-negative supermartingale, and so is their average $E_j(n_k)$.
\end{proof}


Balanced rounds are what make the average conditional mean of each round equal $R(\lambda_j)$; a round
drawn only from low-risk tasks would have an increment with expectation above one.
 
\subsection{Identification: Why Paired, Complete Episodes Are Needed}
\label{app:identification}
 
\paragraph{Marginal success rates.}
Let $s_0=\Pr(Y_0{=}1)$, $s_\lambda=\Pr(Y_\lambda{=}1)$, and $p_{ab}=\Pr(Y_0{=}a,Y_\lambda{=}b)$ under the
task mixture, so $R(\lambda)=p_{10}$. Fixing the marginals leaves one free parameter, $p_{11}$, which can
take any value in $[\max\{0,s_0+s_\lambda-1\},\min\{s_0,s_\lambda\}]$. Hence $R(\lambda)=s_0-p_{11}$ can be
any value in $[\max\{0,s_0-s_\lambda\},\,\min\{s_0,1-s_\lambda\}]$. For example, with $s_0=s_\lambda=0.9$,
any $R(\lambda)\in[0,0.1]$ is possible. The success gap satisfies $s_0-s_\lambda=R(\lambda)-p_{01}$: episodes
the candidate repairs ($p_{01}$) cancel episodes it breaks.
 
\paragraph{Individual decisions.}
Consider an accelerator that replaces the reference action with an accelerated action at a set of
decisions $\mathcal{T}_\lambda\subseteq\{0,\ldots,H-1\}$. Its \emph{single-decision effect} at
$t\in\mathcal{T}_\lambda$ is the AIF loss of the controller that accelerates decision $t$ alone and follows
the reference elsewhere.
 
\begin{theorem}[Single-decision effects do not determine AIF risk]
\label{thm:single_decision_nonid}
There exist two accelerators on the same task, using the same accelerated action, whose single-decision
effects are all zero, yet whose AIF risks are $1$ and $0$.
\end{theorem}
 
\begin{proof}
Let the state be a scalar error $x_t\ge0$ with $x_0=0$, and let an episode of $H\ge2$ decisions succeed
if and only if $x_H\le1$. The reference action sets $x_{t+1}=\max\{x_t-1,0\}$, so the reference keeps
$x_t\equiv0$ and succeeds. The accelerated action sets $x_{t+1}=x_t+1$. Accelerating a single decision
raises the error to $1$, after which reference actions return it to $0$ (or leave $x_H=1$ if $t=H-1$), so
every single-decision effect is zero. Accelerating every decision gives $x_H=H\ge2$ and fails, so its risk
is $1$. Accelerating every other decision keeps $x_t\le1$ throughout and succeeds, so its risk is $0$.
\end{proof}
 
The fully accelerated controller reaches states (persistent error) that no single-decision intervention
from the reference trajectory reaches, which is why AIF risk must be measured on complete episodes.
 
\subsection{Retention Bounds and Retention Budgets}
\label{app:retention_theory}
 
Let $s_0 := \Pr(Y_0=1)$ and assume $s_0>0$. Retention is $\operatorname{Ret}(\lambda)=1-\rho$ with $\rho=R(\lambda)/\Pr(Y_0{=}1)\in[0,1]$. Because it is
a ratio of two unknown quantities, we bound it by inverting a family of tests rather than dividing
estimates.

\begin{theorem}[Retention bound]
\label{thm:retention_zbeta}
For $\beta\in(0,1]$ let $Z_\beta(U)=\bigl(L_\lambda(U)-\beta Y_0(U)+\beta\bigr)/(1+\beta)$, and let
$\widehat p(\beta)=\widehat p\bigl(\widehat Z_\beta;\frac{\beta}{1+\beta}\bigr)$ as in
Eq.~\ref{eq:hb_pvalue}. Let $\mathcal{S}=\{\beta\in(0,1]:\widehat p(\beta)>\delta'\}$ (with
$\sup\emptyset=0$) and let $\bar b\ge\sup\mathcal{S}$. Then
$\Pr[\operatorname{Ret}(\lambda)\ge1-\bar b]\ge1-\delta'$. With $\delta'=\delta/M$ for every candidate, the
bounds hold simultaneously, and hence for $\widehat\lambda$, with probability at least $1-\delta$.
\end{theorem}
 
\begin{proof}
Since $L_\lambda\le Y_0$, the pair $(Y_0,L_\lambda)$ is $(0,0)$, $(1,0)$, or $(1,1)$, giving
$Z_\beta=\frac{\beta}{1+\beta}$, $0$, or $\frac{1}{1+\beta}$, so $Z_\beta\in[0,1]$. Taking expectations,
$\mathbb{E}[Z_\beta]=\bigl(R(\lambda)-\beta\Pr(Y_0{=}1)+\beta\bigr)/(1+\beta)$, which equals
$\frac{\beta}{1+\beta}$ exactly when $\beta=\rho$. If $\rho=0$, the claim holds since $\bar b\ge0$. If
$\rho>0$, the test at $\beta=\rho$ has a true null, so $\Pr(\widehat p(\rho)\le\delta')\le\delta'$
(App.~\ref{app:assumptions}). Hence, with probability at least $1-\delta'$, $\rho\in\mathcal{S}$, so
$\rho\le\bar b$. The simultaneous statement follows from the union bound.
\end{proof}
 
\paragraph{Computing $\bar b$.}
$\widehat p(\beta)$ is not monotone in $\beta$, because the ceiling in the binomial term jumps as
$n\widehat Z_\beta$ crosses integers, so a bisection can return an invalid bound. Instead, we exclude
$\beta$ values where either term of Eq.~\ref{eq:hb_pvalue} is at most $\delta'$. Since
$\widehat Z_\beta$ is monotone in $\beta$, the integer crossings have closed form, and between them the
binomial term is monotone; the exponential term is excluded on a grid with an explicit Lipschitz
constant. Scanning down from $\beta=1$, $\bar b$ is the upper end of the first cell that cannot be
excluded, so $\bar b\ge\sup\mathcal{S}$.

\paragraph{Certifying a retention budget.}
For a retention budget $1-\beta$, let $\theta_\beta=\beta/(1+\beta)$ and test, for each candidate,
$H^{\mathrm{Ret}}_j:\operatorname{Ret}(\lambda_j)\le1-\beta$. By the proof of Theorem~3,
$H^{\mathrm{Ret}}_j$ holds if and only if $\mathbb{E}[Z_\beta]\ge\theta_\beta$, a mean test on the
$[0,1]$-valued loss $Z_\beta$. The fixed-sample test uses $\widehat p(\widehat Z_\beta;\theta_\beta)$,
which is valid by App.~\ref{app:assumptions}. The sequential test uses Eq.~\ref{eq:eprocess} with
$\alpha$ replaced by $\theta_\beta$ and $f_j(n_k)$ by $\sum_{i\le n_k}Z_{\beta,j,i}$;
Lemma~\ref{lem:round_supermartingale} still applies (see its proof), and early exit remains valid since
$Z_\beta\ge0$. Theorem~\ref{thm:risk_control} then holds with $\operatorname{Ret}$ in place of $R$: with
probability at least $1-\delta$, $\operatorname{Ret}(\widehat\lambda)>1-\beta$, and the fallback has
$\operatorname{Ret}(\lambda_0)=1$. Since $Z_\beta$ depends on $Y_0$ even when the candidate succeeds, this
mode runs the reference once on every calibration unit instead of using failure-triggered evaluation.
 
\paragraph{Retention under failure-triggered evaluation.}
The bound above needs $Y_0$ on every unit, which failure-triggered evaluation does not provide. In that
mode we instead run the reference on $n_0$ calibration units designated in advance, whose empirical
success rate $\widehat s_0$ gives, by Hoeffding's inequality, $\Pr(Y_0{=}1)\ge\ell_0=\widehat
s_0-\sqrt{\ln(1/\delta_0)/(2n_0)}$ with probability at least $1-\delta_0$. Combined with
$R(\widehat\lambda)<\alpha$ (Theorem~\ref{thm:risk_control}), the union bound gives
$\operatorname{Ret}(\widehat\lambda)>1-\alpha/\ell_0$ with probability at least $1-\delta-\delta_0$
whenever $\ell_0>0$ (and $\operatorname{Ret}=1$ if $\widehat\lambda=\lambda_0$). These reference outcomes
are cached and reused during certification. We use $\delta_0=0.01$ and split the remaining $0.04$ across
the thirty-two suite--candidate hypotheses, so the combined statement holds with probability at least
$0.95$.

\section{Additional Experimental Results}
\label{app:experiments}

\subsection{Design-Choice Validation and Cost Ablations}
\label{app:design_analysis}

This section provides the complete analyses underlying
Section~\ref{sec:ablation}. We first examine the three choices required
for valid certification and then isolate the two mechanisms that reduce
evaluation cost.

\vspace{-6pt}
\paragraph{Episode-level outcomes.}
We first compare accelerating a single policy decision with applying
acceleration throughout the complete episode. The former measures a local
intervention, whereas the latter measures the deployment risk targeted by
\ours{}. We show the results in Table \ref{tab:local_episode}. 

\begin{table}[H]
\centering
\caption{\textbf{Local acceleration effects do not determine
whole-episode AIF.} The relationship varies across suites because an
accelerated action can alter subsequent states, observations, and actions.}
\label{tab:local_episode}
\vspace{3pt}
\begin{tabular}{lcc}
\toprule
Suite &
\shortstack{One accelerated\\decision} &
\shortstack{Whole-episode\\AIF} \\
\midrule
Goal    & $2.3\%$ & $2.6\%$ \\
Spatial & $0.7\%$ & $0.8\%$ \\
Object  & $4.3\%$ & $1.0\%$ \\
Long    & $9.5\%$ & $2.7\%$ \\
\bottomrule
\end{tabular}
\vspace{-10pt}
\end{table}

The two quantities are close on Goal and Spatial but differ by factors of
$4.31$ and $3.52$ on Object and Long. Their relationship is therefore not
stable across deployments. This result supports measuring the terminal
outcomes of complete paired episodes rather than using local confidence or
disagreement signals as a proxy for AIF.

\paragraph{Family correction.}
We next isolate the effect of searching over multiple candidate policies.
We simulate ten candidates with known risks around $\alpha=0.1$, use 200
calibration episodes per trial, and repeat the experiment for 4,000 trials.
Each method selects from the same candidate family and observations. We present the results in Table \ref{tab:multiplicity}. 

\begin{table}[ht]
\centering
\caption{\textbf{Family correction prevents selection-induced risk
inflation.} Searching over ten candidates without correction deploys an
over-budget policy in $46.0\%$ of trials, whereas family-corrected testing
remains at or below $0.15\%$.}
\label{tab:multiplicity}
\vspace{3pt}
\begin{tabular}{lc}
\toprule
Selection rule & Over-budget deployment \\
\midrule
Uncorrected empirical selection & $46.0\%$ \\
Family-corrected testing        & $\leq 0.15\%$ \\
\bottomrule
\end{tabular}
\vspace{-10pt}
\end{table}

The uncorrected rule repeatedly selects candidates whose empirical risks
appear favorable because of sampling variation. Correcting the complete
family prevents this selection effect from invalidating the deployment
guarantee.

\paragraph{Task-balanced acquisition.}
The deployment risk averages over a fixed mixture of tasks. Sequential
certification must therefore preserve that mixture at every interim look.
We compare the balanced round-robin schedule used by \ours{} with
task-by-task acquisition, which completes one task before proceeding to
the next. Each entry in Table~\ref{tab:balanced_order} is the exact
probability of certifying an over-budget candidate.


\begin{table}[H]
\centering
\caption{\textbf{Balanced acquisition preserves sequential validity.}
Balanced rounds remain below the candidate-level error budget for every tested task-risk profile, whereas task-by-task
ordering can certify an over-budget policy before its high-risk tasks are
evaluated.}
\label{tab:balanced_order}
\vspace{3pt}
\resizebox{0.7\linewidth}{!}{
\begin{tabular}{lcc}
\toprule
Task-risk profile &
Balanced rounds &
Task by task \\
\midrule
Eight safe, two risky; $R=0.050$
& $1.84\times10^{-4}$ & $1.00$ \\
Graded; $R=0.050$
& $4.37\times10^{-4}$ & $1.42\times10^{-2}$ \\
Homogeneous; $R=0.050$
& $4.77\times10^{-4}$ & $4.77\times10^{-4}$ \\
One high-risk task; $R=0.060$
& $1.08\times10^{-6}$ & $1.00$ \\
Six safe, four risky; $R=0.052$
& $2.43\times10^{-4}$ & $1.00$ \\
\bottomrule
\end{tabular}
}
\end{table}

Balanced acquisition remains valid even when AIF risk varies sharply
across tasks. Task-by-task ordering can initially expose only low-risk
tasks and certify a policy before observing the tasks responsible for its
true deployment risk.

\paragraph{Evaluation-cost components.}
Finally, we separate the two mechanisms used by \ours{} to reduce
certification cost. We present the results in Table \ref{tab:cost_ablation}. Sequential fastest-first evaluation tests policies in
increasing order of measured inference cost and stops a candidate when
certification is no longer attainable. Failure-triggered evaluation runs
H1 only when a candidate fails, because the paired AIF outcome is already
zero whenever the candidate succeeds.

\begin{table}[ht]
\centering
\caption{\textbf{Component-wise ablation of certification cost.}
Sequential fastest-first evaluation reduces candidate rollouts, while
failure-triggered evaluation reduces H1 rollouts. Combining both components
reduces total rollouts by $77.8\%$ and wall-clock time by $81.9\%$ relative
to fixed-$n$ test-all evaluation.}
\label{tab:cost_ablation}
\vspace{3pt}
\resizebox{0.88\linewidth}{!}{
\begin{tabular}{lrrrr}
\toprule
Procedure &
\shortstack{Candidate\\rollouts} &
\shortstack{H1\\rollouts} &
\shortstack{Total\\rollouts} &
\shortstack{Wall-clock\\time} \\
\midrule
Fixed-$n$ test-all
& $8{,}000$ & $1{,}000$ & $9{,}000$ & $38.6$ h \\
Fixed-$n$ + failure-triggered H1
& $8{,}000$ & $349$ & $8{,}349$ & $33.1$ h \\
Sequential fastest-first + full H1
& $1{,}950$ & $1{,}000$ & $2{,}950$ & $16.3$ h \\
\textbf{\ours}
& $\mathbf{1{,}950}$ & $\mathbf{47}$ &
$\mathbf{1{,}997}$ & $\mathbf{7.0}$ h \\
\bottomrule
\end{tabular}
}
\end{table}

\begin{table}[ht]
\centering
\caption{\textbf{Sensitivity to risk budget and calibration size.}
A larger risk budget or more calibration evidence reduces fallback and
increases speedup, while \ours{} maintains an observed over-budget rate
of at most $0.02\%$.}
\label{tab:sensitivity}
\vspace{2pt}
\footnotesize
\setlength{\tabcolsep}{1pt}
\begin{tabular}{ccccc}
\toprule
$\alpha (\%)$ & $n$ &
\shortstack{Over budget (\%)} &
\shortstack{Fallback (\%)} &
\shortstack{Speedup ($\times$)} \\
\midrule
2.0 & 250 & 0.00 & 100.0 & 1.00 \\
2.5 & 250 & 0.00 & 100.0 & 1.00 \\
3.0 & 250 & 0.00 & 69.5  & 3.51 \\
5.0 & 250 & 0.00 & 11.7  & 8.88 \\
10.0 & 250 & 0.00 & 0.0   & 11.25 \\
\midrule
2.5 & 500   & 0.02 & 56.3 & 4.59 \\
2.5 & 1,000 & 0.01 & 15.0 & 8.48 \\
2.5 & 2,000 & 0.00 & 1.4  & 9.99 \\
\bottomrule
\end{tabular}
\vspace{-5pt}
\end{table}

Sequential fastest-first evaluation provides the largest individual
saving, reducing candidate rollouts by $75.6\%$. Failure-triggered
evaluation provides a complementary saving by reducing reference
evaluation to 47 rollouts after sequential candidate stopping. Together,
the two components produce the full $77.8\%$ rollout reduction achieved
by \ours{}.

\subsection{Sensitivity Analysis}
\label{app:sensitivity}

We examine how the risk budget $\alpha$ and calibration size $n$ affect
\ours{} for sensitivity analysis. We report over-budget deployment, fallback to H1, and mean speedup relative to H1.

Table~\ref{tab:sensitivity} shows the results. At the default setting ($\alpha=5.0\%$, $n=250$), \ours{} achieves an $8.88\times$ mean speedup with $11.7\%$ fallback and
no observed over-budget deployment. Tightening $\alpha$ increases fallback,
whereas relaxing it allows faster policies to be certified. Under the
stricter $\alpha=0.025$ budget, increasing $n$ from 500 to 2,000 reduces
fallback from $56.3\%$ to $1.4\%$ and raises mean speedup from
$4.59\times$ to $9.99\times$.  Thus, \ours{} becomes less
conservative as the budget or calibration evidence grows, while maintaining risk control. The default setting provides a practical
operating point considering evaluation cost, while \ours{} adapts robustly to other risk budgets and calibration sizes.

\section{Benchmarks}
\label{app:benchmarks}

\paragraph{LIBERO.}
LIBERO~\citep{liu2023libero} is a simulated tabletop manipulation benchmark in which a Franka Panda arm
follows language instructions. We use its four ten-task suites (Figure~\ref{fig:libero_tasks}):
\emph{Spatial} varies object layouts, \emph{Object} varies the target object, \emph{Goal} varies the
goal in a fixed scene, and \emph{Long} chains several subgoals. Each task provides $50$ fixed initial
states, giving $500$ scenes per suite, and an episode succeeds if the task's goal conditions are met
before the step limit.

\paragraph{Crafter.}
Crafter~\citep{hafner2021benchmarking} is a procedurally generated 2D survival game with discrete
actions and a status bar tracking health, food, drink, energy, and inventory
(Figure~\ref{fig:crafter_tasks}). We use two achievements as tasks, \texttt{collect\_wood} and
\texttt{collect\_drink}, and an episode succeeds if the agent unlocks the achievement within the step
budget. It tests whether \ours{} applies unchanged to language-model agents.
\begin{figure}[t]
\centering
\includegraphics[width=\linewidth]{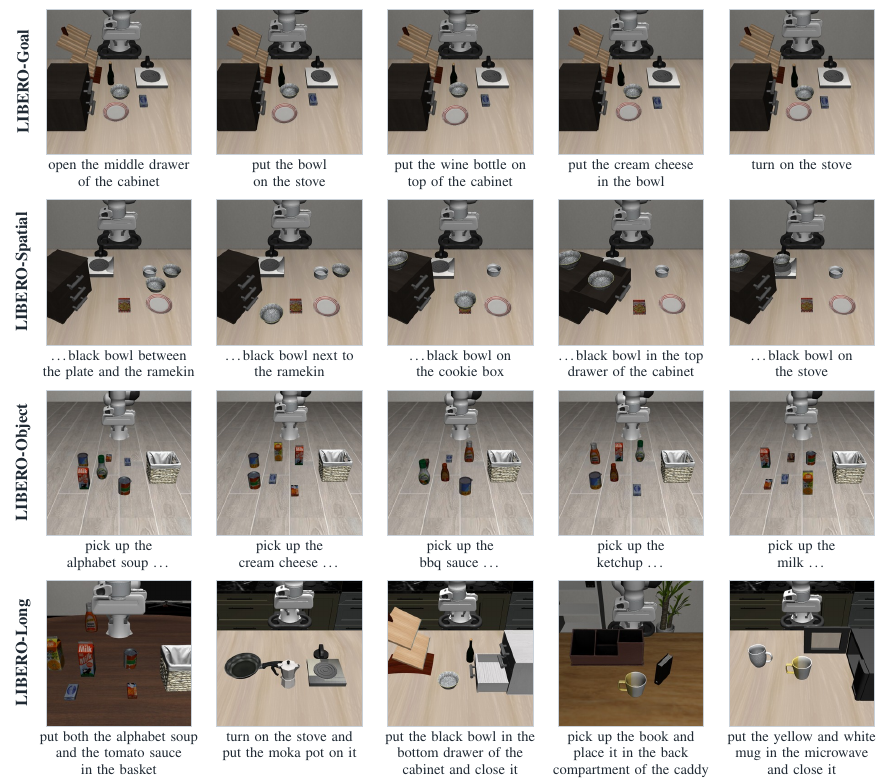}
\caption{\textbf{LIBERO task suites.} Initial scenes of five of the ten tasks in each suite, rendered from the third-person agent-view camera, with the language instruction below each scene. Spatial instructions have the form ``pick up the black bowl \ldots\ and place it on the plate'' and Object instructions ``pick up the \ldots\ and place it in the basket''; only the distinguishing part is shown. Each evaluation unit fixes a task, one initial scene, and the random seeds.}
\label{fig:libero_tasks}
\end{figure}
 
\begin{figure}[t]
\centering
\includegraphics[width=\linewidth]{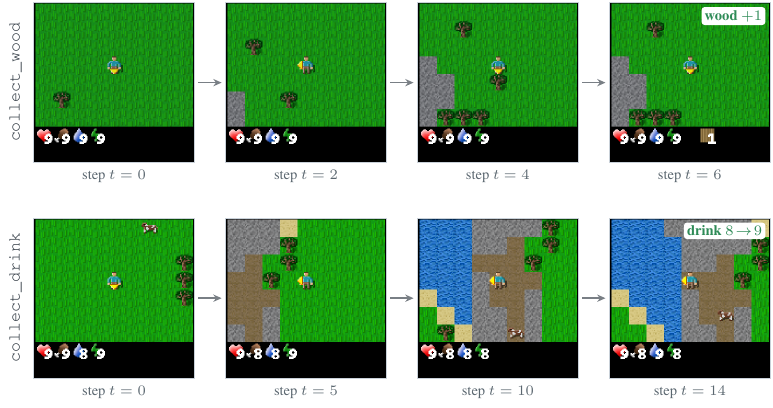}
\caption{\textbf{Crafter tasks.} Example episodes of \texttt{collect\_wood} (top) and
\texttt{collect\_drink} (bottom). The status bar shows health, food, drink, and energy, followed by the
inventory; an episode succeeds once wood enters the inventory or the agent drinks water. Frames are
rendered from the environment with a scripted policy for illustration and are not rollouts of the
evaluated language-model agents.}
\label{fig:crafter_tasks}
\end{figure}


\end{document}